\documentclass{article}

\usepackage{longtable}
\usepackage[utf8]{inputenc} 
\usepackage[T1]{fontenc}    
\usepackage{hyperref}       
\usepackage{url}            
\usepackage{booktabs}       
\usepackage{amsfonts}       
\usepackage{nicefrac}       
\usepackage{microtype}      
\usepackage{paralist}
\usepackage{floatrow}
\usepackage{amsmath}
\usepackage{graphicx}
\usepackage{adjustbox}
\usepackage{color}
\usepackage{fullpage}
\usepackage{amsthm}
\newfloatcommand{capbtabbox}{table}[][\FBwidth]
\usepackage{authblk}

\newtheorem{observation}{Observation}

\newcommand{\wvec}[1]{\overrightarrow{\text{#1}}}

\newcommand{\ignore}[1]{}

\title{Man is to Computer Programmer as Woman is to Homemaker?\\Debiasing Word Embeddings}

\author[1]{{\normalsize Tolga Bolukbasi}}
\author[2]{{\normalsize Kai-Wei Chang}}
\author[2]{{\normalsize James Zou}}
\author[1,2]{{\normalsize Venkatesh Saligrama}}
\author[2]{{\normalsize Adam Kalai}}
\affil[1]{{\small Boston University, 8 Saint Mary's Street, Boston, MA}}
\affil[2]{{\small Microsoft Research New England, 1 Memorial Drive, Cambridge, MA}}
\affil[ ]{{\small tolgab@bu.edu, kw@kwchang.net, jamesyzou@gmail.com, srv@bu.edu, adam.kalai@microsoft.com}\vspace{-0.8cm}}

\date{}
\begin{document}
\maketitle
\begin{abstract}
The blind application of machine learning runs the risk of amplifying biases present in data. Such a danger is facing us with {\em word embedding}, a popular framework to represent text data as vectors which has been used in many machine learning and natural language processing tasks. We show that even word embeddings trained on Google News articles exhibit female/male gender stereotypes to a disturbing extent. This raises concerns because their widespread use, as we describe, often tends to amplify these biases. Geometrically, gender bias is first shown to be captured by a direction in the word embedding. Second, gender neutral words are shown to be linearly separable from gender definition words in the word embedding. Using these properties, we provide a methodology for modifying an embedding to remove gender stereotypes, such as the association between between the words {\em receptionist} and {\em female}, while maintaining desired associations such as between the words {\em queen} and {\em female}. We define metrics to quantify both direct and indirect gender biases in embeddings, and develop algorithms to ``debias'' the embedding. Using crowd-worker evaluation as well as standard benchmarks, we empirically demonstrate that our algorithms significantly reduce gender bias in embeddings while preserving the its useful properties such as the ability to cluster related concepts and to solve analogy tasks. The resulting embeddings can be used in applications without amplifying gender bias.
\vspace{-0.05cm}
\end{abstract}
\section{Introduction}

There have been hundreds or thousands of papers written about word embeddings and their applications, from Web search \cite{NMCC16} to parsing Curriculum Vitae \cite{resumes2015}. However, none of these papers have recognized how blatantly sexist the embeddings are and hence risk introducing biases of various types into real-world systems. 

A word embedding that represent each word (or common phrase) $w$ as a $d$-dimensional {\em word vector} $\vec{w}\in \mathbb{R}^d$. Word embeddings, trained only on word co-occurrence in text corpora, serve as a dictionary of sorts for computer programs that would like to use word meaning. First, words with similar semantic meanings tend to have vectors that are close together. Second, the vector differences between words in embeddings have been shown to represent relationships between words ~\cite{rubenstein1965contextual, mikolov2013linguistic}. For example given an analogy puzzle, ``man is to king as woman is to $x$'' (denoted as \emph{man}:\emph{king} :: \emph{woman}:$x$), simple arithmetic of the embedding vectors finds that $x$=\emph{queen} is the best answer because:
\[
\overrightarrow{\text{man}}-\overrightarrow{\text{woman}}\approx \overrightarrow{\text{king}}-\overrightarrow{\text{queen}}
\]
 Similarly, $x$=\emph{Japan} is returned for \emph{Paris}:\emph{France} ::  
\emph{Tokyo}:$x$. It is surprising that a simple vector arithmetic can simultaneously capture a variety of relationships. It has also excited practitioners because such a tool could be useful across applications involving natural language. Indeed, they are being studied and used in a variety of downstream applications (e.g., document ranking~\cite{NMCC16}, sentiment analysis~\cite{IrsoyCardie14}, and question retrieval~\cite{LJBJTMM16}).  

\begin{figure}
\begin{tabular}{lllc}
\multicolumn{3}{c}{\bf{Extreme \emph{she} occupations}}\\ 
{1. homemaker} & {2. nurse} & {3. receptionist} \\
{4. librarian} & {5. socialite} & {6. hairdresser} \\
{7. nanny} & {8. bookkeeper} & {9. stylist} \\
{10. housekeeper} & {11. interior designer} & {12. guidance counselor} \\[2ex]
\multicolumn{3}{c}{{\bf Extreme \emph{he} occupations}}\\
{1. maestro} & {2. skipper} & {3. protege} \\
{4. philosopher} & {5. captain} & {6. architect} \\
{7. financier} & {8. warrior} & {9. broadcaster} \\
{10. magician} & {11. figher pilot} & {12. boss} \\
\\\end{tabular}
\caption{\label{fig:occupation_words} The most extreme occupations as projected on to the {\em she$-$he} gender direction on g2vNEWS. Occupations such as {\em businesswoman}, where gender is suggested by the orthography, were excluded.}
\end{figure}

\begin{figure}
\begin{tabular}{lllc}
\multicolumn{3}{c}{\bf{Gender stereotype \emph{she}-\emph{he} analogies.}}\\ 
{sewing-carpentry} & {register-nurse-physician } & {housewife-shopkeeper } \\
{nurse-surgeon} & {interior designer-architect} & {softball-baseball } \\
{blond-burly} & {feminism-conservatism } & {cosmetics-pharmaceuticals} \\
{giggle-chuckle} & {vocalist-guitarist} & {petite-lanky } \\
{sassy-snappy} & {diva-superstar} & {charming-affable } \\
{volleyball-football} & {cupcakes-pizzas } & {hairdresser-barber } \\[2ex]
\multicolumn{3}{c}{{\bf Gender appropriate \emph{she}-\emph{he} analogies.}}\\
{queen-king} & {sister-brother} & {mother-father} \\
{waitress-waiter} & {ovarian cancer-prostate cancer} & {convent-monastery} \\
\\\end{tabular}
\caption{\textbf{Analogy examples}. Examples of automatically generated analogies for the pair \emph{she-he} using the procedure described in text. For example, the first analogy is interpreted as \emph{she}:\emph{sewing} :: \emph{he}:\emph{carpentry} in the original w2vNEWS embedding. Each automatically generated analogy is evaluated by 10 crowd-workers are to whether or not it reflects gender stereotype. Top: illustrative gender stereotypic analogies automatically generated from w2vNEWS, as rated by at least 5 of the 10 crowd-workers. Bottom: illustrative generated gender-appropriate analogies.  }
\label{fig:analogy_examples}
\end{figure}

\begin{figure}
\begin{tabular}{lll}
   {\bf {\em softball} extreme}	&{\bf gender portion} &{\bf after debiasing}	\\
   1. pitcher	&-1\% 	& 1. pitcher      \\
   2. bookkeeper	&20\%	&2. infielder		\\
   3. receptionist	&67\%	&3. major leaguer		\\
   4. registered nurse	&29\%	&4. bookkeeper		\\
   5. waitress	&35\%	&5. investigator		\\ [2ex]

   {\bf {\em football} extreme}	&{\bf gender portion} &{\bf after debiasing} \\
   1. footballer	&2\%	&1. footballer	\\
   2. businessman	&31\%	&2. cleric	\\
   3. pundit	&10\%	&3. vice chancellor	\\
   4. maestro	&42\%	&4. lecturer	\\
   5. cleric	&2\%	&5. midfielder	\\
\end{tabular}
\caption{\textbf{Example of indirect bias}. The five most extreme occupations on the {\em softball-football} axis, which indirectly captures gender bias. For each occupation, the degree to which the association represents a gender bias is shown, as described in Section \ref{sec:indirect}.
\label{fig:occupation_associations2}
}
\end{figure}

However, the embeddings also pinpoint sexism implicit in text. For instance, it is also the case that:
\[
\overrightarrow{\text{man}}-\overrightarrow{\text{woman}}\approx \overrightarrow{\text{computer programmer}}-\overrightarrow{\text{homemaker}}.
\]
In other words, the same system that solved the above reasonable analogies will offensively answer ``man is to computer programmer as woman is to $x$'' with $x$=\textit{homemaker}. Similarly, it outputs that a {\em father} is to a {\em doctor} as a \textit{mother} is to a \textit{nurse}. The primary embedding studied in this paper is the popular publicly-available word2vec~\cite{MCCD13,MSCCD13} embedding trained on a corpus of Google News texts consisting of 3 million English words and terms into 300 dimensions, which we refer to here as the w2vNEWS. One might have hoped that the Google News embedding would exhibit little gender bias because many of its authors are professional journalists. We also analyze other publicly available embeddings trained via other algorithms and find similar biases.

In this paper, we will quantitatively demonstrate that word-embeddings contain biases in their geometry that reflect gender stereotypes present in broader society. Due to their wide-spread usage as basic features, word embeddings not only reflect such stereotypes but can also amplify them. This poses a significant risk and challenge for machine learning and its applications. 

To illustrate bias amplification, consider bias present in the task of retrieving relevant web pages for a given query. In web search, one recent project has shown that, when carefully combined with existing approaches, word vectors have the potential to improve web page relevance results \cite{NMCC16}. As an example, suppose the search query is {\em cmu computer science phd student} for a computer science Ph.D.\ student at Carnegie Mellon University. Now, the directory\footnote{Graduate Research Assistants listed at \url{http://cs.cmu.edu/directory/csd}.} offers 127 nearly identical web pages for students --- these pages differ only in the names of the students. A word embedding's semantic knowledge can improve relevance by identifying, for examples, that the terms {\em graduate research assistant} and {\em phd student} are related. However, word embeddings also rank terms related to computer science closer to male names than female names (e.g., the embeddings give \emph{John}:\emph{computer programmer} :: \emph{Mary}:{\em homemaker}). The consequence is that, between two pages that differ only in the names \emph{Mary} and \emph{John}, the word embedding would influence the search engine to rank John's web page higher than Mary. In this hypothetical example, the usage of word embedding makes it even harder for women to be recognized as computer scientists and would contribute to widening the existing gender gap in computer science. While we focus on gender bias, specifically Female-Male (F-M) bias, the approach may be applied to other types of bias.  

Uncovering gender stereotypes from text may seem like a trivial matter of counting pairs of words that occur together. However, such counts are often misleading \cite{gordon2013reporting}. For instance, the term {\em male nurse} is several times more frequent than {\em female nurse} (similarly {\em female quarterback} is many times more frequent than {\em male quarterback}). Hence, extracting associations from text, F-M or otherwise, is not simple, and ``first-order'' approaches would predict that the word {\em nurse} is more male than {\em quarterback}. More generally, Gordon and Van Durme show how {\em reporting bias} \cite{gordon2013reporting}, including the fact that common assumptions are often left unsaid, poses a challenge to extracting knowledge from raw text. Nonetheless, $\wvec{nurse}$ is closer to $\wvec{female}$ than $\wvec{male}$, suggesting that word embeddings may be capable of circumventing reporting bias in some cases.  This happens because word embeddings are trained using second-order methods which require large amounts of data to extract associations and relationships about words. 

The analogies generated from these embeddings spell out the bias implicit in the data on which they were trained. Hence, word embeddings may serve as a means to extract implicit gender associations from a large text corpus similar to how Implicit Association Tests \cite{greenwald1998measuring} detect automatic gender associations possessed by people, which often do not align with self reports.

To quantify bias, we compare a word embedding to the embeddings of a pair of gender-specific words. For instance, the fact that $\wvec{nurse}$ is close to $\wvec{woman}$ is not in itself necessarily biased (it is also somewhat close to $\wvec{man}$ -- all are humans), but the fact that these distances are unequal suggests bias. To make this rigorous, consider the distinction between {\em gender specific} words that are associated with a gender by definition, and the remaining {\em gender neutral} words. Standard examples of gender specific words include \emph{brother}, \emph{sister}, {\em businessman} and {\em businesswoman}. 
The fact that $\wvec{brother}$ is closer to $\wvec{man}$ than to $\wvec{woman}$ is expected since they share the definitive feature of relating to males. We will use the gender specific words to learn a gender subspace in the embedding, and our debiasing algorithm removes the bias only from the gender neutral words while respecting the definitions of these gender specific words. 

We refer to this type of bias, where there is an association between a gender neutral word and a clear gender pair as {\em direct bias}. We also consider a notion of {\em indirect bias},\footnote{The terminology indirect bias follows Pedreshi et al.~\cite{pedreshi2008discrimination} who distinguish {\em direct} versus {\em indirect} discrimination in rules of fair {\em classifiers}. Direct discrimination involves directly using sensitive features such as gender or race, whereas indirect discrimination involves using correlates that are not inherently based on sensitive features but that, intentionally or unintentionally, lead to disproportionate treatment nonetheless.} which manifests as associations between gender neutral words that are clearly arising from gender. For instance, the fact that the word {\em receptionist} is much closer to {\em softball} than {\em football} may arise from female associations with both {\em receptionist} and {\em softball}. Note that many pairs of male-biased (or female-biased) words have legitimate associations having nothing to do with gender. For instance, while the words {\em mathematician} and {\em geometry} both have a strong male bias, their similarity is justified by factors other than gender. More often than not, associations are combinations of gender and other factors that can be difficult to disentangle. Nonetheless, we can use the geometry of the word embedding to determine the degree to which those associations are based on gender.

\smallskip
\noindent
{\bf Aligning biases with stereotypes.} Stereotypes are biases that are widely held among a group of people. We show that the biases in the word embedding are in fact closely aligned with social conception of gender stereotype, as evaluated by U.S.-based crowd workers on Amazon's Mechanical Turk.\footnote{\url{http://mturk.com}} The crowd agreed that the biases reflected both in the location of vectors (e.g. $\wvec{doctor}$ closer to $\wvec{man}$ than to $\wvec{woman}$) as well as in analogies (e.g., \emph{he}:\emph{coward} :: \emph{she}:\emph{whore}) exhibit common gender stereotypes.

\smallskip
\noindent
{\bf Debiasing.} 
Our goal is to reduce gender biases in the word embedding while preserving the useful properties of the embedding. 
Surprisingly, not only does the embedding capture bias, but it also contains sufficient information to reduce this bias, as illustrated in \ref{fig:words}. We will leverage the fact that there exists a low dimensional subspace in the embedding that empirically captures much of the gender bias. 
The goals of debiasing are:
\begin{enumerate}
\item Reduce bias:
\begin{enumerate}
\item Ensure that gender neutral words such as {\em nurse} are equidistant between gender pairs such as {\em he} and {\em she}.
\item Reduce gender associations that pervade the embedding even among gender neutral words. 
\end{enumerate}
\item Maintain embedding utility:
\begin{enumerate}
\item Maintain meaningful non-gender-related associations between gender neutral words, including associations within stereotypical categories of words such as fashion-related words or words associated with football.
\item Correctly maintain definitional gender associations such as between {\em man} and {\em father}.
\end{enumerate}
\end{enumerate}

\paragraph{Paper outline.} 
After discussing related literature, we give preliminaries necessary for understanding the paper in Section \ref{sec:defs}. Next we propose methods to identify the gender bias of an embedding and show that w2vNEWS exhibits bias which is aligned with common gender stereotypes (Section~\ref{sec:empiricalbias}). In Section~\ref{sec:bias}, we define several simple geometric properties associated with bias, and in particular discuss how to identify the gender subspace. Using these geometric properties, we introduce debiasing algorithms (Section~\ref{sec:debias}) and demonstrate their performance (Section~\ref{sec:debiasing_results}). Finally we conclude with additional discussions of related literature, other types of biases in the embedding and future works. 

\section{Related work}
Related work can be divided into relevant literature on bias in language and bias in algorithms. 

\subsection{Gender bias and stereotype in English}
It is important to quantify and understand bias in languages as such biases can reinforce the psychological status of different groups \cite{sapir1985selected}.
Gender bias in language has been studied over a number of decades in a variety of contexts (see, e.g., \cite{holmes2008handbook}) and we only highlight some of the findings here. Biases differ across people though commonalities can be detected. Implicit Association Tests \cite{greenwald1998measuring} have uncovered gender-word biases that people do not self-report and may not even be aware of. Common biases link female terms with liberal arts and family and male terms with science and careers \cite{nosek2002harvesting}. Bias is seen in word morphology, i.e., the fact that words such as {\em actor} are, by default, associated with the dominant class \cite{jakobson1990language}, and female versions of these words, e.g., {\em actress}, are marked. There is also an imbalance in the number of words with F-M with various associations. For instance, while there are more words referring to males, there are many more words that sexualize females than males \cite{stanley1977paradigmatic}. 

Glick and Fiske \cite{glick1996ambivalent} introduce the notion of {\em benevolent sexism} in which women are perceived with positive traits such as helpful or intimacy-seeking. Despite its seemingly positive nature, benevolent sexism can be harmful, insulting, and discriminatory. In terms of words, female gender associations with any word, even a subjectively positive word such as {\em attractive}, can cause discrimination against women if it reduces their association with other words, such as {\em professional}.  

Stereotypes, as mentioned, are biases that are widely held within a group. While gender bias of any kind is concerning, stereotypes are often easier to study due to their consistent nature. Stereotypes have commonalities across cultures, though there is some variation between cultures \cite{cuddy2015men}.  {\em Complimentary stereotypes} are common between females and males, in which each gender is associated with strengths that are perceived to offset its own weaknesses and compliment the strengths of the other gender \cite{jost2005exposure}. These and compensatory stereotypes are used by people to justify the status quo.

Consistent biases have been studied within online contexts and specifically related to the contexts we study such as online news (e.g., \cite{ross2011women}), Web search (e.g., \cite{kay2015unequal}), and Wikipedia (e.g., \cite{wikipediaBias}). In Wikipedia, Wager et al.~\cite{wikipediaBias} found that, as suggested by prior work on gender bias in language \cite{finkbeiner2013}, articles about women more often emphasize their gender, their husbands and their husbands' jobs, and other topics discussed consistently less often than in articles about men. Regarding individual words, they find that certain words are predictive of gender, e.g.,  {\em husband} appears significantly more often in articles about women while {\em baseball} occurs more often in articles about men.  

\subsection{Bias within algorithms}
A number of online systems have been shown to exhibit various biases, such as racial discrimination and gender bias in the ads presented to users \cite{sweeney2013discrimination,datta2015automated}. A recent study found that algorithms used to predict repeat offenders exhibit indirect racial biases \cite{angwin2016}. Different demographic and geographic groups also use different dialects and word-choices in social media \cite{diffusionlexical}. An implication of this effect is that language used by minority group might not be able to be processed by natural language tools that are trained on ``standard'' data-sets. Biases in the curation of machine learning data-sets have explored in \cite{Torralba12,annotationbias}. 

Independent from our work, Schmidt \cite{schmidt2015} identified the bias present in word embeddings and proposed debiasing by entirely removing multiple gender dimensions, one for each gender pair. His goal and approach, similar but simpler than ours, was to entirely remove gender from the embedding. There is also an intense research agenda focused on improving the quality of word embeddings from different angles (e.g., \cite{levyGoldberg14,pennington2014glove,yogatamalearning,faruqui2015retrofitting}), and the difficulty of evaluating embedding quality (as compared to supervised learning) parallels  the difficulty of defining bias in an embedding.  

Within machine learning, a body of notable work has focused on ``fair'' binary classification in particular. A definition of fairness based on legal traditions is presented by Barocas and Selbst \cite{barocas2014big}. Approaches to modify classification algorithms to define and achieve various notions of fairness have been described in a number of works, see, e.g., \cite{barocas2014big, dwork2012fairness,feldman2015certifying} and a recent survey \cite{zliobaite2015survey}.

Feldman et al.~\cite{feldman2015certifying} distinguish classification algorithms that achieve fairness by modifying the underlying data from those that achieve fairness by modifying the classification algorithm. Our approach is more similar to the former. However, it is unclear how to apply any of these previous approaches without a clear classification task in hand, and the problem is exacerbated by indirect bias.  

This prior work on algorithmic fairness is largely for supervised learning. Fair classification is defined based on the fact that algorithms were classifying a set of individuals using a set of features with a distinguished sensitive feature. In word embeddings, there are no clear individuals and no a priori defined classification problem. However, similar issues arise, such as direct and indirect bias~\cite{pedreshi2008discrimination}.

\section{Preliminaries}\label{sec:defs}

We first very briefly define an embedding and some terminology. An embedding consists of a unit vector $\vec{w}\in \mathbb{R}^d$, with $\|\vec{w}\|=1$, for each word (or term) $w\in W$. We assume there is a set of gender neutral words $N \subset W$, such as {\em flight attendant} or {\em shoes}, which, by definition, are not specific to any gender. We denote the size of a set $S$ by $|S|$. We also assume we are given a set of F-M gender pairs $P\subset W \times W$, such as 
{\em she-he} or {\em mother-father} whose definitions differ mainly in gender. Section \ref{sec:gender_neutral} discusses how $N$ and $P$ can be found within the embedding itself, but until then we take them as given. 

As is common, {\em similarity} between words $w_1$ and $w_2$ is measured by their inner product, $\vec{w}_1 \cdot \vec{w}_2.$ Finally, we will abuse terminology and refer to the embedding of a word and the word interchangeably. For example, the statement {\em cat} is more similar to {\em dog} than to {\em cow} means $\wvec{cat}\cdot \wvec{dog} \geq \wvec{cat}\cdot \wvec{cow}$. For arbitrary vectors $u$ and $v$, define:
$$\cos(u,v) = \frac{u \cdot v}{\|u\|\|v\|}.$$
This normalized similarity between vectors $u$ and $v$ is written as $\cos$ because it is the cosine of the angle between the two vectors. 
Since words are normalized $\cos(\vec{w}_1,\vec{w}_2) = \vec{w}_1 \cdot \vec{w}_2$.

\smallskip
\noindent
{\bf Embedding.}
Unless otherwise stated, the embedding we refer to in this paper is the aforementioned w2vNEWS embedding, a $d=300$-dimensional word2vec~\cite{MCCD13,MSCCD13} embedding, which has proven to be immensely useful since it is high quality, publicly available, and easy to incorporate into any application. In particular, we downloaded the pre-trained embedding on the Google News corpus,\footnote{\url{https://code.google.com/archive/p/word2vec/}} and normalized each word to unit length as is common. Starting with the 50,000 most frequent words, we selected only lower-case words and phrases consisting of fewer than 20 lower-case characters (words with upper-case letters, digits, or punctuation were discarded). After this filtering, 26,377 words remained. While we focus on w2vNEWS, we show later that gender stereotypes are also present in other embedding data-sets. 

\smallskip
\noindent
{\bf Crowd experiments.}
All human experiments were performed on the Amazon Mechanical Turk crowdsourcing platform. We selected for U.S.-based workers to maintain homogeneity and reproducibility to the extent possible with crowdsourcing. Two types of experiments were performed: ones where we solicited words from the crowd (to see if the embedding biases contain those of  the crowd) and ones where we solicited ratings on words or analogies generated from our embedding (to see if the crowd's biases contain those from the embedding). These two types of experiments are analogous to experiments performed in rating results in information retrieval to evaluate precision and recall. When we speak of the majority of 10 crowd judgments, we mean those annotations made by 5 or more independent workers.

Since gender associations vary by culture and person, we ask for ratings of stereotypes rather than bias. In addition to possessing greater consistency than biases, people may feel more comfortable rating the stereotypes of their culture than discussing their own gender biases. The Appendix contains the questionnaires that were given to the crowd-workers to perform these tasks.

\section{Gender stereotypes in word embeddings}\label{sec:empiricalbias}

Our first task is to understand the biases present in the word-embedding (i.e. which words are closer to \emph{she} than to \emph{he}, etc.) 
and the extent to which these geometric biases agree with human notion of gender stereotypes. We use two simple methods to approach this problem: 1) evaluate whether the embedding has stereotypes on occupation words and 2) evaluate whether the embedding produces analogies that are judged to reflect stereotypes by humans. The exploratory analysis of this section will motivate the more rigorous metrics used in the next two sections.

\paragraph{Occupational stereotypes.} Figure \ref{fig:occupation_words} lists the occupations that are closest to \emph{she} and to \emph{he} in the w2vNEWS embeddings. We asked the crowdworkers to evaluate whether an occupation is considered female-stereotypic, male-stereotypic, or neutral. Each occupation word was evaluated by ten crowd-workers as to whether or not it reflects gender stereotype. Hence, for each word we had a integer rating, on a scale of 0-10, of stereotypicality. The projection of the occupation words onto the \emph{she}-\emph{he} axis is strongly correlated with the stereotypicality estimates of these words (Spearman $\rho = 0.51$), suggesting that the geometric biases of embedding vectors is aligned with crowd judgment of gender stereotypes. We used occupation words here because they are easily interpretable by humans and often capture common gender stereotypes. Other word sets could be used for this task. Also note that we could have used other words, e.g. \emph{woman} and \emph{man}, as the gender-pair in the task. We chose \emph{she} and \emph{he} because they are frequent and do not have fewer alternative word senses (e.g., {\em man} can also refer to {\em mankind}).

We projected each of the occupations onto the \emph{she-he} direction in the w2vNEWS embedding as well as a different embedding generated by the GloVe algorithm on a web-crawl corpus \cite{pennington2014glove}. The results are highly consistent (Figure \ref{fig:other_embeddings}), suggesting that gender stereotypes is prevalent across different embeddings and is not an artifact of the particular training corpus or methodology of word2vec. 

\begin{figure}
\includegraphics[width=2.7in]{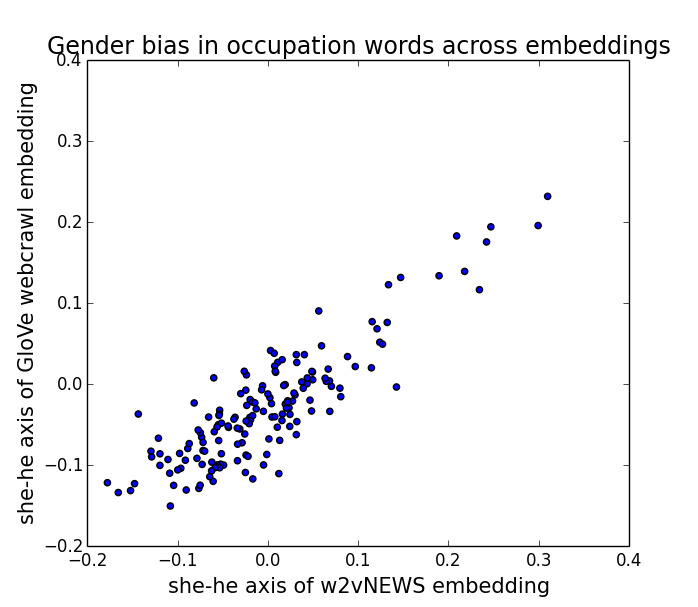}
\caption{\label{fig:other_embeddings}Comparing the bias of two different embeddings--the w2vNEWS and the GloVe web-crawl embedding. In each embedding, the occupation words are projected onto the  \emph{she}-\emph{he} direction. Each dot corresponds to one occupation word; the gender bias of occupations is highly consistent across embeddings (Spearman $\rho = 0.81$). }
\end{figure}

\paragraph{Analogies exhibiting stereotypes.} Analogies are a useful way to both evaluate the quality of a word embedding and also its stereotypes. We first briefly describe how the embedding generate analogies and then discuss how we use analogies to quantify gender stereotype in the embedding. A more detailed discussion of our algorithm and prior analogy solvers is given in Appendix \ref{ap:analogies}.

In the standard analogy tasks, we are given three words, for example \emph{he, she, king}, and look for the 4th word to solve \emph{he} to \emph{king} is as \emph{she} to $x$. Here we modify the analogy task so that given two words, e.g. \emph{he, she}, we want to generate a pair of words, $x$ and $y$, such that \emph{he} to $x$ as \emph{she} to $y$ is a good analogy. This modification allows us to systematically generate pairs of words that the embedding believes it analogous to \emph{he, she} (or any other pair of seed words).

The input into our analogy generator is a seed pair of words $(a,b)$ determining a {\em seed direction} $\vec{a}-\vec{b}$ corresponding to the normalized difference between the two seed words. In the task below, we use $(a,b)=(\text{she}, \text{he})$. We then score all pairs of words $x, y$ by the following metric:
\begin{equation}
\label{eq:analogyS}
\mbox{S}_{(a,b)}(x, y) = \begin{cases}
    \cos\left(\vec{a}-\vec{b}, \vec{x}-\vec{y}\right) & \text{if } \left\|\vec{x}-\vec{y}\right\| \leq \delta\\
    0              & \text{otherwise}
\end{cases}
\end{equation}
where $\delta$ is a threshold for similarity. The intuition of the scoring metric is that we want a good analogy pair to be close to parallel to the seed direction while the two words are not too far apart in order to be semantically coherent. The parameter $\delta$ sets the threshold for semantic similarity. In all the experiments, we take $\delta=1$ as we find that this choice often works well in practice. Since all embeddings are normalized, this threshold corresponds to an angle $\leq \pi/3$, indicating that the two words are closer to each other than they are to the origin. In practice, it means that the two words forming the analogy are significantly closer together than two random embedding vectors.
Given the embedding and seed words, we output the top analogous pairs with the largest positive $S_{(a,b)}$ scores. To reduce redundancy, we do not output multiple analogies sharing the same word $x$.

Since analogies, stereotypes, and biases are heavily influenced by culture, we employed U.S. based crowd-workers to evaluate the analogies output by the analogy generating algorithm described above. For each analogy, we asked the workers two yes/no questions: (a) whether the pairing makes sense as an analogy, and (b) whether it reflects a gender stereotype. Every analogy is judged by 10 workers, and we used the number of workers that rated this pair as stereotyped to quantify the degree of bias of this analogy. Overall, 72 out of 150 analogies were rated as gender-appropriate by five or more crowd-workers, and 29 analogies were rated as exhibiting gender stereotype by five or more crowd-workers (Figure~\ref{fig:direct-result}). Examples of analogies generated from w2vNEWS that were rated as stereotypical are shown at the top of Figure~\ref{fig:analogy_examples}, and examples of analogies that make sense and are rated as gender-appropriate are shown at the bottom of Figure~\ref{fig:analogy_examples}. The full list of analogies and crowd ratings are in Appendix \ref{app:analogies}.

\paragraph{Indirect gender bias.} The direct bias analyzed above manifests in the relative similarities between gender-specific words and gender neutral words. Gender bias could also affect the relative geometry between gender neutral words themselves. To test this \emph{indirect} gender bias, we take pairs of words that are gender-neutral, for example \emph{softball} and \emph{football}. We project all the occupation words onto the $\wvec{softball} - \wvec{football}$ direction and looked at the extremes words, which are listed in Figure~\ref{fig:occupation_associations2}. 
For instance, the fact that the words {\em bookkeeper} and {\em receptionist} are much closer to {\em softball} than {\em football} may result indirectly from female associations with {\em bookkeeper}, {\em receptionist} and {\em softball}. It's important to point out that 
 that many pairs of male-biased (or female-biased) words have legitimate associations having nothing to do with gender. For example, while both {\em footballer} and {\em football} have strong male biases, their similarity is justified by factors other than gender. In Section~\ref{sec:bias}, we define a metric to more rigorously quantify these indirect effects of gender bias.

\section{Geometry of Gender and Bias}\label{sec:bias}

In this section, we study the bias present in the embedding geometrically, identifying the gender direction and quantifying the bias independent of the extent to which it is aligned with the crowd bias. We develop metrics of direct and indirect bias that more rigorously quantify the observations of the previous section. 

\subsection{Identifying the gender subspace}\label{sec:direction}

Language use is ``messy'' and therefore individual word pairs do not always behave as expected. For instance, the word {\em man} has several different usages: it may be used as an exclamation as in {\em oh man!} or to refer to people of either gender or as a verb, e.g., {\em man the station}. To more robustly estimate bias, we shall aggregate across multiple paired comparisons. By combining several directions, such as $\wvec{she}-\wvec{he}$ and $\wvec{woman} - \wvec{man}$, we identify a {\bf gender direction} $g\in \mathbb{R}^d$ that largely captures gender in the embedding. This direction helps us to quantify direct and indirect biases in words and associations.

\begin{figure}
\begin{tabular}{r@{\hskip -0.0in}lllr@{\hskip -0.0in}lll}
& & def. & stereo. & & & def. & stereo. \\
$\wvec{she}-$&$\wvec{he}$ & 92\% & 89\% & $\wvec{daughter}-$&$\wvec{son}$ & 93\% & 91\% \\
$\wvec{her}-$&$\wvec{his}$ & 84\% & 87\% & $\wvec{mother}-$&$\wvec{father}$ & 91\% & 85\% \\
$\wvec{woman}-$&$\wvec{man}$ & 90\% & 83\% & $\wvec{gal}-$&$\wvec{guy}$ & 85\% & 85\% \\
$\wvec{Mary}-$&$\wvec{John}$ & 75\% & 87\% & $\wvec{girl}-$&$\wvec{boy}$ & 90\% & 86\% \\
$\wvec{herself}-$&$\wvec{himself}$ & 93\% & 89\% & $\wvec{female}-$&$\wvec{male}$ & 84\% & 75\%  
\end{tabular}
\caption{\label{fig:definitional_pairs} 
Ten possible word pairs to define gender, ordered by word frequency, along with agreement with two sets of 100 words solicited from the crowd, one with definitional and and one with stereotypical gender associations. For each set of words, comprised of the most frequent 50 female and 50 male crowd suggestions, the accuracy is shown for the corresponding gender classifier based on which word is closer to a target word, e.g., the {\em she-he} classifier predicts a word is female if it is closer to {\em she} than {\em he}. With roughly 80-90\% accuracy, the gender pairs predict the gender of both stereotypes and definitionally gendered words solicited from the crowd. }
\end{figure}

In English as in many languages, there are numerous gender pair terms, and for each we can consider the difference between their embeddings. Before looking at the data, one might imagine that they all had roughly the same vector differences, as in the following caricature: 
\begin{eqnarray*}
\wvec{grandmother} &=& \wvec{wise} + \wvec{gal}\\
\wvec{grandfather} &=& \wvec{wise} + \wvec{guy}\\
\wvec{grandmother} - \wvec{grandfather} &=& \wvec{gal}  - \wvec{guy} = g
\end{eqnarray*}
However, gender pair differences are not parallel in practice, for multiple reasons. First, there are different biases associated with with different gender pairs. Second is polysemy, as mentioned, which in this case occurs due to the other use of {\em grandfather} as in {\em to grandfather a regulation}. Finally, randomness in the word counts in any finite sample will also lead to differences.  Figure \ref{fig:definitional_pairs} illustrates ten possible gender pairs, $\bigl\{(x_i,y_i)\bigr\}_{i=1}^{10}$.

We experimentally verified that the pairs of vectors corresponding to these words do agree with the crowd concept of gender. On Amazon Mechanical Turk, we asked crowdworkers to generate two lists of words: one list corresponding to words that they think are gendered by definition ({\em waitress}, {\em menswear}) and a separate list corresponding to words that they believe captures gender stereotypes (e.g., {\em sewing}, {\em football}). From this we generated the most frequently suggested  50 male and 50 female words for each list to be used for a classification task. For each candidate pair, for example $\wvec{she}, \wvec{he}$, we say that it accurately classifies a crowd suggested female definition (or stereotype) word if that word vector is closer to $\wvec{she}$ than to $\wvec{he}$. Table~\ref{fig:definitional_pairs} reports the classification accuracy for definition and stereotype words for each gender pair. The accuracies are high, indicating that these pairs capture the intuitive notion of gender. 

To identify the gender subspace, we took the ten gender pair difference vectors and computed its principal components (PCs). As Figure \ref{fig:PCA} shows, there is a single direction that explains the majority of variance in these vectors. The first eigenvalue is significantly larger than the rest. Note that, from the randomness in a finite sample of ten noisy vectors, one expects a decrease in eigenvalues. However, as also illustrated in \ref{fig:PCA}, the decrease one observes due to random sampling is much more gradual and uniform. Therefore we hypothesize that the top PC, denoted by the unit vector $g$, captures the gender subspace. In general, the gender subspace could be higher dimensional and all of our analysis and algorithms (described below) work with general subspaces.

\begin{figure}
\begin{tabular}{cc}
\includegraphics[width=2in]{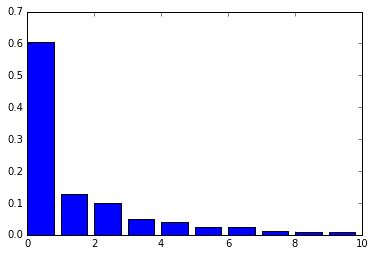} & \includegraphics[width=2in]{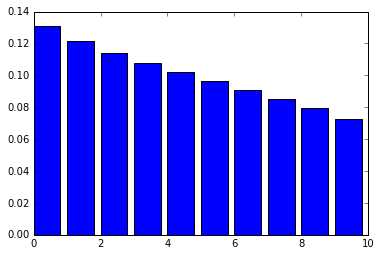}
\end{tabular}
\caption{\label{fig:PCA} Left: the percentage of variance explained in the PCA of these vector differences  (each difference normalized to be a unit vector). The top component explains significantly more variance than any other. 
Right: for comparison, the corresponding percentages for random unit vectors (figure created by averaging over 1,000 draws of ten random unit vectors in 300 dimensions).
}
\end{figure}

\subsection{Direct bias}
To measure direct bias, we first identify words that should be gender-neutral for the application in question. How to generate this set of gender-neutral words is described in  Section \ref{sec:gender_neutral}.
Given the gender neutral words, denoted by $N$, and the gender direction learned from above, $g$, we define the direct gender bias of an embedding to be
\[
\text{DirectBias}_c = \frac{1}{|N|} \sum_{w \in N} \left|\cos(\vec{w}, g)  \right|^c
\]
where $c$ is a parameter that determines how \emph{strict} do we want to in measuring bias. If $c$ is 0, then $\left|\cos(\vec{w} - g)  \right|^c = 0$ only if $\vec{w}$ has no overlap with $g$ and otherwise it is 1. Such strict measurement of bias might be desirable in settings such as the college admissions example from the Introduction, where it would be unacceptable for the embedding to introduce a slight preference for one candidate over another by gender. A more gradual bias would be setting $c = 1$.  The presentation we have chosen favors simplicity -- it would be natural to extend our definitions to weight words by frequency. For example, in w2vNEWS, if we take $N$ to be the set of 327 occupations, then $\text{DirectBias}_1=0.08$, which confirms that many occupation words have substantial component along the gender direction. 

\subsection{Indirect bias}\label{sec:indirect}
Unfortunately, the above definitions still do not capture indirect bias. To see this, imagine completely removing from the embedding both words in gender pairs (as well as words such as {\em beard} or {\em uterus} that are arguably gender-specific but which cannot be paired). There would still be  indirect gender association in that a word that should be gender neutral, such as {\em receptionist}, is closer to {\em softball} than {\em football} (see Figure \ref{fig:occupation_associations2}). As discussed in the Introduction, it can be subtle to obtain the ground truth of the extent to which such similarities is due to gender. 

The gender subspace $g$ that we have identified allows us to quantify the contribution of $g$ to the similarities between any pair of words. We can decompose a given word vector $w \in \mathbb{R}^d$ as $w = w_g + w_{\perp}$, where $w_g = (w \cdot g) g$ is the contribution from gender and $w_{\perp} = w - w_g$. Note that all the word vectors are normalized to have unit length. We define the gender component to the similarity between two word vectors $w$ and $v$ as 
\[
\beta(w, v) = \left( w\cdot v - \frac{w_{\perp}\cdot v_{\perp}}{\|w_{\perp}\|_2 \|v_{\perp}\|_2}  \right)\bigg/w\cdot v.
\]

The intuition behind this metric is as follow: $\frac{w_{\perp}\cdot v_{\perp}}{\|w_{\perp}\|_2 \|v_{\perp}\|_2}$ is the inner product between the two vectors if we project out the gender subspace and renormalize the vectors to be of unit length. The metric quantifies how much this inner product changes (as a fraction of the original inner product value) due to this operation of removing the gender subspace. Because of noise in the data, every vector has some non-zero component $w_{\perp}$ and $\beta$ is well-defined. Note that $\beta(w,w) = 0$, which is reasonable since the similarity of a word to itself should not depend on gender contribution. If $w_g = 0 = v_g$, then $\beta(w,v) = 0$; and if $w_{\perp} = 0 = v_{\perp}$, then $\beta(w, v) = 1$. 

In Figure~\ref{fig:occupation_associations2}, as a case study, we examine the most extreme words on the $\wvec{softball} - \wvec{football}$ direction. The five most extreme words (i.e. words with the highest positive or the lowest negative projections onto  $\wvec{softball} - \wvec{football}$) are shown in the table. Words such as \emph{receptionist}, \emph{waitress} and \emph{homemaker} are closer to \emph{softball} than \emph{football}, and the $\beta$'s between these words and \emph{softball} is substantial (67\%, 35\%, 38\%, respectively). This suggests that the apparent similarity in the embeddings of these words to $\wvec{softball}$ can be largely explained by gender biases in the embedding. Similarly, \emph{businessman} and \emph{maestro} are closer to \emph{football} and this can also be attributed largely to indirect gender bias, with $\beta$'s of 31\% and 42\%, respectively.

\begin{figure}
\includegraphics[width=\linewidth,trim={0.5in 0 0 0},clip]{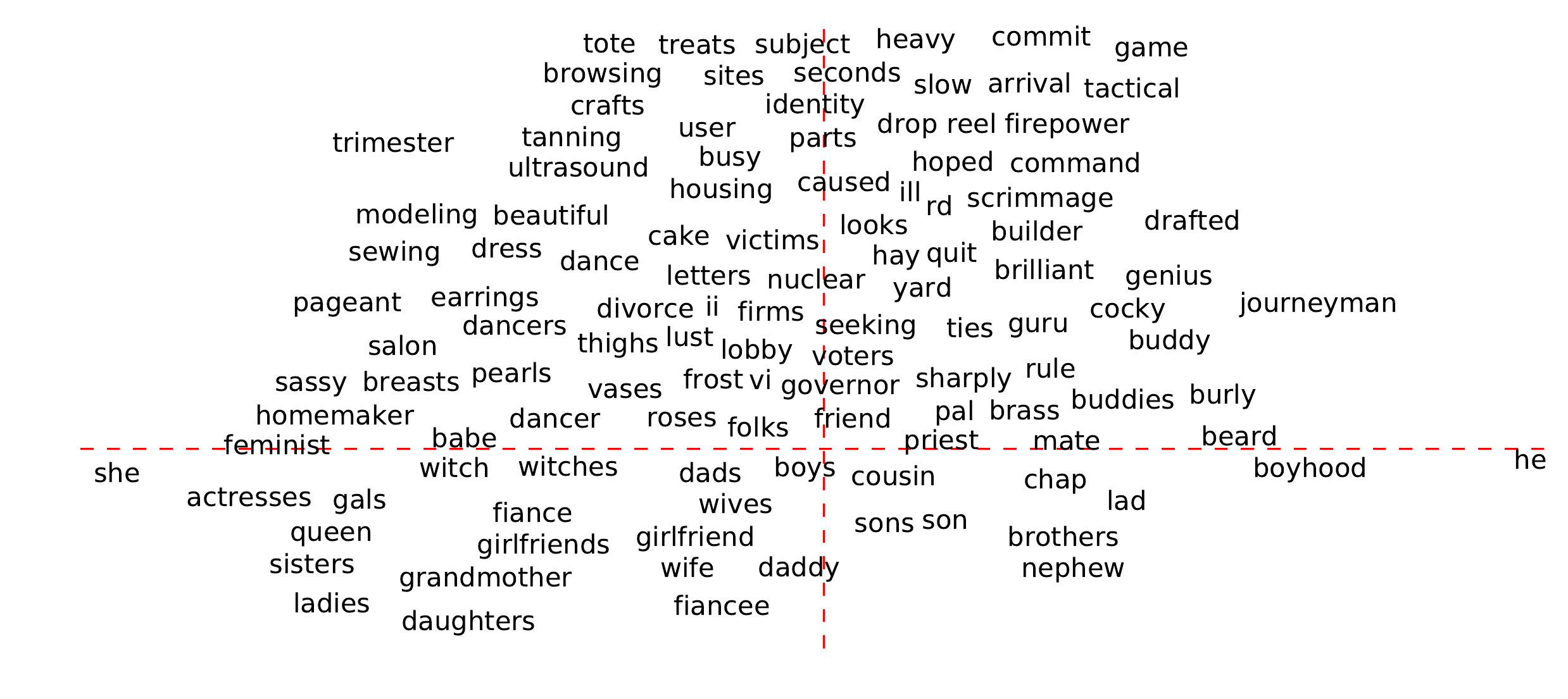}
\caption{Selected words projected along two axes: $x$ is a projection onto the difference between the embeddings of the words {\em he} and {\em she}, and $y$ is a direction learned in the embedding that captures gender neutrality, with gender neutral words above the line and gender specific words below the line. Our hard debiasing algorithm removes the gender pair associations for gender neutral words. In this figure, the words above the horizontal line would all be collapsed to the vertical line.
\label{fig:words}}
\end{figure}

\section{Debiasing algorithms}\label{sec:debias}
The debiasing algorithms are defined in terms of sets of words rather than just pairs, for generality, so that we can consider other biases such as racial or religious biases. We also assume that we have a set of words to neutralize, which can come from a list or from the embedding as described in Section \ref{sec:gender_neutral}. (In many cases it may be easier to list the gender specific words not to neutralize as this set can be much smaller.)

The first step, called {\bf Identify gender subspace}, is to identify a direction (or, more generally, a subspace) of the embedding that captures the bias. For the second step, we define two options: {\bf Neutralize and Equalize} or {\bf Soften}. {\bf Neutralize} ensures that gender neutral words are zero in the gender subspace. {\bf Equalize} perfectly equalizes sets of words outside the subspace and thereby enforces the property that any neutral word is equidistant to all words in each equality set. For instance, if $\{\text{grandmother}, \text{grandfather}\}$ and $\{\text{guy}, \text{gal}\}$ were two equality sets, then after equalization {\em babysit} would be equidistant to {\em grandmother} and {\em grandfather} and also equidistant to {\em gal} and {\em guy}, but presumably closer to the grandparents and further from the {\em gal} and {\em guy}. This is suitable for applications where one does not want any such pair to display any bias with respect to neutral words. 

The disadvantage of Equalize is that it removes certain distinctions that are valuable in certain applications. For instance, one may wish a language model to assign a higher probability to the phrase {\em to grandfather a regulation}) than {\em to grandmother a regulation} since {\em grandfather} has a meaning that {\em grandmother} does not -- equalizing the two removes this distinction. The Soften algorithm reduces the differences between these sets while maintaining as much similarity to the original embedding as possible, with a parameter that controls this trade-off. 

To define the algorithms, it will be convenient to introduce some further notation. A subspace $B$ is defined by $k$ orthogonal unit vectors $B = \{b_1, \ldots, b_k\} \subset \mathbb{R}^d$. In the case $k=1$, the subspace is simply a direction. We denote the projection of a vector $v$ onto $B$ by,
$$v_B = \sum_{j=1}^k (v \cdot b_j) b_j.$$
This also means that $v-v_B$ is the projection onto the orthogonal subspace.

\smallskip
\noindent
{\bf Step 1: Identify gender subspace}. Inputs: word sets $W$,  defining sets $D_1, D_2, \ldots, D_n \subset W$ as well as embedding $\bigr\{\vec{w} \in \mathbb{R}^d\bigl\}_{w \in W}$ and integer parameter $k \geq 1$. Let $$\mu_i := \sum_{w \in D_i} \vec{w}/|D_i|$$ be the means of the defining sets. Let the bias subspace $B$ be the first $k$ rows of $\mathrm{SVD}(\textbf{C})$ where $$\textbf{C} := \sum_{i=1}^n \sum_{w \in D_i} (\vec{w}-\mu_i)^T (\vec{w}-\mu_i)\bigl/|D_i|.$$ 

\smallskip

{\bf Step 2a: Hard de-biasing (neutralize and equalize)}. Additional inputs: words to neutralize $N\subseteq W$, family of equality sets $\mathcal{E} = \{E_1, E_2, \ldots, E_m\}$ where each $E_i \subseteq W$. For each word $w\in N$, let $\vec{w}$ be re-embedded to $$\vec{w} := (\vec{w} - \vec{w}_B)\bigl/\|\vec{w} - \vec{w}_B\|.$$ 
For each set $E\in \mathcal{E}$, let 
\begin{eqnarray*}
\mu &:=& \sum_{w \in E} w/|E|\\
\nu &:=& \mu - \mu_B\\
\text{ For each }w \in E, ~~ \vec{w} &:=& \nu + \sqrt{1-\|\nu\|^2}\frac{\vec{w}_B-\mu_B}{\|\vec{w}_B-\mu_B\|}
\end{eqnarray*}
Finally, output the subspace $B$ and the new embedding $\bigr\{\vec{w} \in \mathbb{R}^d\bigl\}_{w \in W}$.

Equalize equates each set of words outside of $B$ to their simple average $\nu$ and then adjusts vectors so that they are unit length. It is perhaps easiest to understand by thinking separately of the two components $\vec{w}_B$ and $\vec{w}_{\perp B} = \vec{w}-\vec{w}_B$. The latter $\vec{w}_{\perp B}$ are all simply equated to their average. Within $B$, they are centered (moved to mean 0) and then scaled so that each $\vec{w}$ is unit length. To motivate why we center, beyond the fact that it is common in machine learning, consider the bias direction being the gender direction ($k=1$) and a gender pair such as $E=\{\text{male}, \text{female}\}$. As discussed, it so happens that both words are positive  (female) in the gender direction, though {\em female} has a greater projection. One can only speculate as to why this is the case, e.g., perhaps the frequency of text such as {\em male nurse} or {\em male escort} or {\em she was assaulted by the male}. However, because {\em female} has a greater gender component, after centering the two will be symmetrically balanced across the origin. If instead, we simply scaled each vector's component in the bias direciton without centering, {\em male} and {\em female} would have exactly the same embedding and we would lose analogies such as {\em father:male :: mother:female}.

Before defining the Soften alternative step, we note that Neutralizing and Equalizing completely remove pair bias.
\begin{observation}
After Steps 1 and 2a, for any gender neutral word $w$ any equality set $E$, and any two words $e_1, e_2 \in E$, $\vec{w} \cdot \vec{e}_1 = w \cdot \vec{e}_2$ and $\|\vec{w}-\vec{e}_1\| = \|\vec{w}-\vec{e}_2\|$. Furthermore, if $\mathcal{E}=\bigl\{\{x,y\}|(x,y)\in P\bigr\}$ are the sets of pairs defining PairBias, then $\mathrm{PairBias}=0$.
\end{observation}
\begin{proof}
Step 1 ensures that $\vec{w}_B=0$, while step 2a ensures that $\vec{e}_1 - vec{e}_2$ lies entirely in $B$. Hence, their inner product is 0 and 
$\vec{w} \cdot \vec{e}_1 = w \cdot \vec{e}_2$. Lastly, $\|\vec{w}-\vec{e}_1\| = \|\vec{w}-\vec{e}_2\|$ follows from the fact that $\|u_1-u_2\|^2 = 2-2u_1\cdot u_2$ for unit vectors $u_1, u_2$ and PairBias being 0 follows trivially from the definition of PairBias.
\end{proof}

{\bf Step 2b: Soft bias correction}.
Overloading the notation, we let $W \in \mathbb{R}^{d \times |vocab|}$ denote the matrix of all embedding vectors and $N$ denote the matrix of the embedding vectors corresponding to gender neutral words. $W$ and $N$ are learned from some corpus and are inputs to the algorithm. The desired debiasing transformation $T \in \mathbb{R}^{d \times d}$ is a linear transformation that seeks to preserve pairwise inner products between all the word vectors while minimizing the projection of the gender neutral words onto the gender subspace. This can be formalized as the following optimization problem
\[
\min_T  \| (TW)^T(TW)-W^T W \|_F^2 + \lambda \|(TN)^T(TB)\|_F^2
\]
where $B$ is the gender subspace learned in Step 1 and $\lambda$ is a tuning parameter that balances the objective of preserving the original embedding inner products with the goal of reducing gender bias. For $\lambda$ large, $T$ would remove the projection onto $B$ from all the vectors in $N$, which corresponds exactly to Step 2a. In the experiment, we use $\lambda=0.2$. The optimization problem is a semi-definite program and can be solved efficiently. The output embedding is normalized to have unit length, $\hat{W} = \{Tw/\|Tw\|_2, w \in W\}$.

\section{Determining gender neutral words}\label{sec:gender_neutral}
For practical purposes, since there are many fewer gender specific words, it is more efficient to enumerate the set of gender specific words $S$ and take the gender neutral words to be the compliment, $N=W\setminus S$. Using dictionary definitions, we derive a subset $S_0$ of 218 words out of the words in w2vNEWS. Recall that this embedding is a subset of 26,377 words out of the full 3 million words in the embedding, as described in Section \ref{sec:defs}. This base list $S_0$ is given in Appendix \ref{ap:gender_neutral}. Note that the choice of words is subjective and ideally should be customized to the application at hand. 

We generalize this list to the entire 3 million words in the Google News embedding using a linear classifier, resulting in the set $S$ of 6,449 gender-specific words. More specifically, we trained a linear Support Vector Machine (SVM) with the default regularization parameter of $C=1.0$. We then ran this classifier on the remaining words, taking $S=S_0\cup S_1$, where $S_1$ were the words labeled as gender specific by our classifier among the words in the entire embedding that were not in the 26,377 words of w2vNEWS.

Using 10-fold cross-validation to evaluate the accuracy of this process, we find an $F$-score of $.627 \pm .102$ based on stratified 10-fold cross-validation. The binary accuracy is well over 99\% due to the imbalanced nature of the classes. For another test of how accurately the embedding agrees with our base set of 218 words, we evaluate the class-balanced error by re-weighting the examples so that the positive and negative examples have equal weights, i.e., weighting each class inverse proportionally to the number of samples from that class. Here again, we use stratified 10-fold cross validation to evaluate the error. Within each fold, the regularization parameter was also chosen by 10-fold (nested) cross validation. The average (balanced) accuracy of the linear classifiers, across folds, was $95.12\% \pm 1.46\%$ with 95\% confidence. 

Figure~\ref{fig:words} illustrates the results of the classifier for separating gender-specific words from gender-neutral words. To make the figure legible, we show a subset of the words. The $x$-axis correspond to projection of words onto the $\wvec{she} - \wvec{he}$ direction and the $y$-axis corresponds to the distance from the decision boundary of the trained SVM. 

\section{Debiasing results}\label{sec:debiasing_results}

We evaluated our debiasing algorithms to ensure that they preserve the desirable properties of the original embedding while reducing both direct and indirect gender biases.

\begin{table}
\begin{tabular}{l|lll}
  & RG & WS & analogy  \\
  \hline
  \hline\\
Before  &	62.3& 54.5 & 57.0 \\
Hard-debiased & 62.4	& 54.1 & 57.0\\
Soft-debiased & 62.4 & 54.2 & 56.8 \\
\end{tabular}
	\caption{The columns show the performance of the original 
	w2vNEWS embedding (``before'') and the debiased w2vNEWS on the standard evaluation metrics measuring coherence and analogy-solving abilities: RG \cite{rubenstein1965contextual}, WS \cite{finkelstein2001placing}, MSR-analogy \cite{mikolov2013linguistic}. Higher is better. The results show that the performance does not degrade after debiasing.
	Note that we use a subset of vocabulary in the experiments.   Therefore, 
	the performances are lower than the previously published results.  
	\label{tab:variance}}
\end{table}

\paragraph{Direct Bias.} 
First we used the same analogy generation task as before: 
for both the hard-debiased and the soft-debiased embeddings, we automatically generated pairs of words that are analogous to 
\emph{she-he} and asked crowd-workers to evaluate whether these pairs reflect 
gender stereotypes. Figure \ref{fig:direct-result} shows the results.   
On the initial w2vNEWS embedding, 19\% of the top 150 analogies were judged as 
showing gender stereotypes by a majority of the ten workers.    
After applying our hard debiasing algorithm, only 6\% of the new embedding were 
judged as stereotypical. As an example, consider the analogy puzzle, \emph{he} to \emph{doctor} is as \emph{she} to $X$. The original embedding returns $X = $ \emph{nurse} while the hard-debiased embedding finds $X =$ \emph{physician}. Moreover the hard-debiasing algorithm preserved gender appropriate analogies such as \emph{she} to \emph{ovarian cancer} is as \emph{he} to \emph{prostate cancer}. 
This demonstrates that the hard-debiasing has effectively reduced the gender stereotypes in the word embedding. 
Figure \ref{fig:direct-result} also shows that the number of 
appropriate analogies remains similar as in the original embedding after executing hard-debiasing.  This 
demonstrates that that the quality of the embeddings 
is preserved. The details results are in Appendix \ref{app:analogies}. Soft-debiasing was less effective in removing gender bias. 

To further confirms the quality of embeddings after debiasing,
 we tested the debiased embedding on several standard benchmarks that measure whether related words have similar embeddings  as well as how well the embedding performs in analogy tasks. Table \ref{tab:variance} shows the results on the original and the new embeddings and the transformation does not negatively impact the performance.

\paragraph{Indirect bias.} We also investigated how the strict debiasing 
algorithm affects indirect gender bias. Because we do not have the ground 
truth on the indirect effects of gender bias, it is challenging to quantify 
the performance of the algorithm in this regard. However we do see promising 
qualitative improvements, as shown in Figure~\ref{fig:occupation_associations2} 
in the \emph{softball}, \emph{football} example. After applying the strict 
debias algorithm, we repeated the experiment and show the most extreme words 
in the $\wvec{softball}-\wvec{football}$ direction. The most extreme words 
closer to \emph{softball} are now \emph{infielder} and \emph{major leaguer} in 
addition to \emph{pitcher}, which are more relevant and do not exhibit 
gender bias. Gender stereotypic associations such are \emph{receptionist}, 
\emph{waitress} and \emph{homemaker} are moved down the list. Similarly, words 
that clearly show male bias, e.g. \emph{businessman}, are also no longer at 
the top of the list.  Note that the two most extreme words in the 
$\wvec{softball} - \wvec{football}$ direction are \emph{pitcher} and 
\emph{footballer}. The similarities between \emph{pitcher} and \emph{softball} 
and between \emph{footballer} and \emph{football} comes from the actual 
functions of these words and hence have little gender contribution. These two 
words are essentially unchanged by the debiasing algorithm.

\begin{figure}
\begin{tabular}{cc}
\includegraphics[width=.45\textwidth]{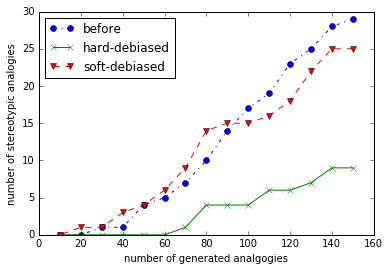} & \includegraphics[width=.45\textwidth]{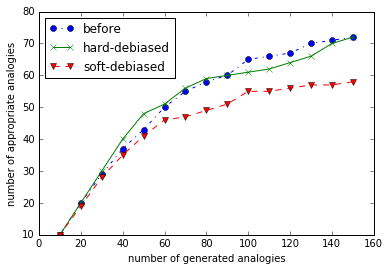}
\end{tabular}
\caption{Number of stereotypical (Left) and appropriate (Right) analogies 
generated by wordembeddings before and after debiasing.  }
\label{fig:direct-result}
\end{figure}

\section{Discussion}

Word embeddings help us further our understanding of bias in language. We find a single direction that largely captures gender, that helps us capture associations between gender neutral  words and gender as well as indirect inequality.The projection of gender neutral words on this direction enables us to quantify their degree of female- or male-bias.

To reduce the bias in an embedding, we change the embeddings of gender neutral words, by removing their gender associations. For instance, {\em nurse} is moved to to be equally male and female in the direction $g$.  In addition, we find that gender-specific words have additional biases beyond $g$. For instance, {\em grandmother} and {\em grandfather} are both closer to {\em wisdom} than {\em gal} and {\em guy} are, which does not reflect a gender difference. On the other hand, the fact that {\em babysit} is so much closer to {\em grandmother} than {\em grandfather} (more than for other gender pairs) is a gender bias specific to {\em grandmother}. By equating {\em grandmother} and {\em grandfather} outside of gender, and since we've removed $g$ from {\em babysit}, both {\em grandmother} and {\em grandfather} and equally close to {\em babysit} after debiasing. By retaining the gender component for gender-specific words, we maintain analogies such as {\em she:grandmother :: he:grandfather}. Through empirical evaluations, we show that our hard-debiasing algorithm significantly reduces both direct and indirect gender bias while preserving the utility of the embedding. We have also developed a soft-embedding algorithm which balances reducing bias with preserving the original distances, and could be appropriate in specific settings. 

One perspective on bias in word embeddings is that it merely reflects bias in society, and therefore one should attempt to debias society rather than word embeddings. However, by reducing the bias in today's computer systems (or at least not amplifying the bias), which is increasingly reliant on word embeddings, in a small way debiased word embeddings can hopefully contribute to reducing gender bias in society. At the very least, machine learning should not be used to inadvertently amplify these biases, as we have seen can naturally happen. 

In specific applications, one might argue that gender biases in the embedding (e.g. \emph{computer programmer} is closer to \emph{he}) could capture useful statistics and that, in these special cases, the original biased embeddings could be used. However given the potential risk of having machine learning algorithms that amplify gender stereotypes and discriminations, we recommend that we should err on the side of neutrality and use the debiased embeddings provided here as much as possible. 

In this paper, we focus on quantifying and reducing gender bias in word embeddings. Corpus of documents often contain other undesirable stereotypes and these can also be reflected in the embedding vectors. The same w2vNEWS also exhibits strong racial stereotype. For example, projecting all the occupation words onto the direction $\wvec{minorities} - \wvec{whites}$, we find that the most extreme occupations closer to \emph{whites} are \emph{parliamentarian}, \emph{advocate}, \emph{deputy}, \emph{chancellor}, \emph{legislator}, and \emph{lawyer}. In contrast, the most extreme occupations at the \emph{minorites} end are \emph{butler}, \emph{footballer}, \emph{socialite}, and \emph{crooner}. It is a subtle issue to understand the direct and indirect bias due to racial, ethnic and cultural stereotypes. An important direction of future work would be to quantify and remove these biases.

While we focus on English word embeddings, it is also an interesting direction to consider how the approach and findings here would apply to other languages, especially languages with grammatical gender where the definitions of most nouns carry a gender marker.

\bibliographystyle{abbrv}
\bibliography{bibfile}

\begin{thebibliography}{10}

\bibitem{angwin2016}
J.~Angwin, J.~Larson, S.~Mattu, and L.~Kirchner.
\newblock Machine bias: There’s software used across the country to predict
  future criminals. and it’s biased against blacks., 2016.

\bibitem{finkbeiner2013}
C.~Aschwanden.
\newblock The finkbeiner test: What matters in stories about women scientists?
\newblock {\em DoubleXScience}, 2013.

\bibitem{barocas2014big}
S.~Barocas and A.~D. Selbst.
\newblock Big data's disparate impact.
\newblock {\em Available at SSRN 2477899}, 2014.

\bibitem{annotationbias}
E.~Beigman and B.~B. Klebanov.
\newblock Learning with annotation noise.
\newblock In {\em ACL}, 2009.

\bibitem{cuddy2015men}
A.~J. Cuddy, E.~B. Wolf, P.~Glick, S.~Crotty, J.~Chong, and M.~I. Norton.
\newblock Men as cultural ideals: Cultural values moderate gender stereotype
  content.
\newblock {\em Journal of personality and social psychology}, 109(4):622, 2015.

\bibitem{datta2015automated}
A.~Datta, M.~C. Tschantz, and A.~Datta.
\newblock Automated experiments on ad privacy settings.
\newblock {\em Proceedings on Privacy Enhancing Technologies}, 2015.

\bibitem{dwork2012fairness}
C.~Dwork, M.~Hardt, T.~Pitassi, O.~Reingold, and R.~Zemel.
\newblock Fairness through awareness.
\newblock In {\em Innovations in Theoretical Computer Science Conference},
  2012.

\bibitem{diffusionlexical}
J.~Eisenstein, B.~O'Connor, N.~A. Smith, and E.~P. Xing.
\newblock Diffusion of lexical change in social media.
\newblock {\em PLoS ONE}, pages 1--13, 2014.

\bibitem{faruqui2015retrofitting}
M.~Faruqui, J.~Dodge, S.~K. Jauhar, C.~Dyer, E.~Hovy, and N.~A. Smith.
\newblock Retrofitting word vectors to semantic lexicons.
\newblock In {\em NAACL}, 2015.

\bibitem{feldman2015certifying}
M.~Feldman, S.~A. Friedler, J.~Moeller, C.~Scheidegger, and
  S.~Venkatasubramanian.
\newblock {Certifying and removing disparate impact}.
\newblock In {\em KDD}, 2015.

\bibitem{fellbaum98}
C.~Fellbaum, editor.
\newblock {\em Wordnet: An Electronic Lexical Database}.
\newblock The MIT Press, Cambridge, MA, 1998.

\bibitem{finkelstein2001placing}
L.~Finkelstein, E.~Gabrilovich, Y.~Matias, E.~Rivlin, Z.~Solan, G.~Wolfman, and
  E.~Ruppin.
\newblock Placing search in context: The concept revisited.
\newblock In {\em WWW}. ACM, 2001.

\bibitem{glick1996ambivalent}
P.~Glick and S.~T. Fiske.
\newblock The ambivalent sexism inventory: Differentiating hostile and
  benevolent sexism.
\newblock {\em Journal of personality and social psychology}, 70(3):491, 1996.

\bibitem{gordon2013reporting}
J.~Gordon and B.~Van~Durme.
\newblock Reporting bias and knowledge extraction.
\newblock {\em Automated Knowledge Base Construction (AKBC)}, 2013.

\bibitem{greenwald1998measuring}
A.~G. Greenwald, D.~E. McGhee, and J.~L. Schwartz.
\newblock Measuring individual differences in implicit cognition: the implicit
  association test.
\newblock {\em Journal of personality and social psychology}, 74(6):1464, 1998.

\bibitem{resumes2015}
C.~Hansen, M.~Tosik, G.~Goossen, C.~Li, L.~Bayeva, F.~Berbain, and M.~Rotaru.
\newblock How to get the best word vectors for resume parsing.
\newblock In {\em SNN Adaptive Intelligence / Symposium: Machine Learning 2015,
  Nijmegen.}

\bibitem{holmes2008handbook}
J.~Holmes and M.~Meyerhoff.
\newblock {\em The handbook of language and gender}, volume~25.
\newblock John Wiley \& Sons, 2008.

\bibitem{IrsoyCardie14}
O.~\.Irsoy and C.~Cardie.
\newblock Deep recursive neural networks for compositionality in language.
\newblock In {\em NIPS}. 2014.

\bibitem{jakobson1990language}
R.~Jakobson, L.~R. Waugh, and M.~Monville-Burston.
\newblock {\em On language}.
\newblock Harvard Univ Pr, 1990.

\bibitem{jost2005exposure}
J.~T. Jost and A.~C. Kay.
\newblock Exposure to benevolent sexism and complementary gender stereotypes:
  consequences for specific and diffuse forms of system justification.
\newblock {\em Journal of personality and social psychology}, 88(3):498, 2005.

\bibitem{kay2015unequal}
M.~Kay, C.~Matuszek, and S.~A. Munson.
\newblock Unequal representation and gender stereotypes in image search results
  for occupations.
\newblock In {\em Human Factors in Computing Systems}. ACM, 2015.

\bibitem{LJBJTMM16}
T.~Lei, H.~Joshi, R.~Barzilay, T.~Jaakkola, A.~M. Katerina~Tymoshenko, and
  L.~Marquez.
\newblock Semi-supervised question retrieval with gated convolutions.
\newblock In {\em NAACL}. 2016.

\bibitem{levyGoldberg14}
O.~Levy and Y.~Goldberg.
\newblock Linguistic regularities in sparse and explicit word representations.
\newblock In {\em CoNLL}, 2014.

\bibitem{MCCD13}
T.~Mikolov, K.~Chen, G.~Corrado, and J.~Dean.
\newblock Efficient estimation of word representations in vector space.
\newblock In {\em ICLR}, 2013.

\bibitem{MSCCD13}
T.~Mikolov, I.~Sutskever, K.~Chen, G.~S. Corrado, and J.~Dean.
\newblock Distributed representations of words and phrases and their
  compositionality.
\newblock In {\em NIPS}.

\bibitem{mikolov2013linguistic}
T.~Mikolov, W.-t. Yih, and G.~Zweig.
\newblock Linguistic regularities in continuous space word representations.
\newblock In {\em HLT-NAACL}, pages 746--751, 2013.

\bibitem{NMCC16}
E.~Nalisnick, B.~Mitra, N.~Craswell, and R.~Caruana.
\newblock Improving document ranking with dual word embeddings.
\newblock In {\em www}, April 2016.

\bibitem{nosek2002harvesting}
B.~A. Nosek, M.~Banaji, and A.~G. Greenwald.
\newblock Harvesting implicit group attitudes and beliefs from a demonstration
  web site.
\newblock {\em Group Dynamics: Theory, Research, and Practice}, 6(1):101, 2002.

\bibitem{pedreshi2008discrimination}
D.~Pedreshi, S.~Ruggieri, and F.~Turini.
\newblock Discrimination-aware data mining.
\newblock In {\em Proceedings of the 14th ACM SIGKDD international conference
  on Knowledge discovery and data mining}, pages 560--568. ACM, 2008.

\bibitem{pennington2014glove}
J.~Pennington, R.~Socher, and C.~D. Manning.
\newblock Glove: Global vectors for word representation.
\newblock In {\em EMNLP}, 2014.

\bibitem{ross2011women}
K.~Ross and C.~Carter.
\newblock Women and news: A long and winding road.
\newblock {\em Media, Culture \& Society}, 33(8):1148--1165, 2011.

\bibitem{rubenstein1965contextual}
H.~Rubenstein and J.~B. Goodenough.
\newblock Contextual correlates of synonymy.
\newblock {\em Communications of the ACM}, 8(10):627--633, 1965.

\bibitem{sapir1985selected}
E.~Sapir.
\newblock {\em Selected writings of Edward Sapir in language, culture and
  personality}, volume 342.
\newblock Univ of California Press, 1985.

\bibitem{schmidt2015}
B.~Schmidt.
\newblock Rejecting the gender binary: a vector-space operation.
\newblock
  \url{http://bookworm.benschmidt.org/posts/2015-10-30-rejecting-the-gender-binary.html},
  2015.

\bibitem{stanley1977paradigmatic}
J.~P. Stanley.
\newblock Paradigmatic woman: The prostitute.
\newblock {\em Papers in language variation}, pages 303--321, 1977.

\bibitem{sweeney2013discrimination}
L.~Sweeney.
\newblock Discrimination in online ad delivery.
\newblock {\em Queue}, 11(3):10, 2013.

\bibitem{Torralba12}
A.~Torralba and A.~Efros.
\newblock Unbiased look at dataset bias.
\newblock In {\em CVPR}, 2012.

\bibitem{turney2012domain}
P.~D. Turney.
\newblock Domain and function: A dual-space model of semantic relations and
  compositions.
\newblock {\em Journal of Artificial Intelligence Research}, pages 533--585,
  2012.

\bibitem{wikipediaBias}
C.~Wagner, D.~Garcia, M.~Jadidi, and M.~Strohmaier.
\newblock It's a man's wikipedia? assessing gender inequality in an online
  encyclopedia.
\newblock In {\em Ninth International AAAI Conference on Web and Social Media},
  2015.

\bibitem{yogatamalearning}
D.~Yogatama, M.~Faruqui, C.~Dyer, and N.~A. Smith.
\newblock Learning word representations with hierarchical sparse coding.
\newblock In {\em ICML}, 2015.

\bibitem{zliobaite2015survey}
I.~Zliobaite.
\newblock A survey on measuring indirect discrimination in machine learning.
\newblock {\em arXiv preprint arXiv:1511.00148}, 2015.

\end{thebibliography}
\appendix

\section{Generating analogies}\label{ap:analogies}

We now expand on different possible methods for generating $(x,y)$ pairs, given $(a,b)$ for generating analogies $a$:$x$ :: $b$:$y$. The first and simplest metric is to consider scoring an analogy by $\|(\vec{a}-\vec{b})-(\vec{x}-\vec{y})$. This may be called the {\em parallelogram} approach and, for the purpose of finding the best single $y$ given $a,b,x$, it is equivalent to the most common approach to finding single word analogies, namely maximizing $\cos(\vec{y}, \vec{x}+\vec{b}-\vec{a})$ called {\em cosAdd} in earlier work \cite{mikolov2013linguistic} since we assume all vectors are unit length. This works well in some cases, but a weakness can be seen that, for many triples $(a,b,x)$, the closest word to $x$ is $y=x$, i.e., $x = \arg\min_{y} \|(\vec{a}-\vec{b})-(\vec{x}-\vec{y})\|$. As a result, the definition explicitly excludes the possibility of returning $x$ itself. In these cases, $y$ is often a word very similar to $x$, and in most of these cases such an algorithm produces two opposing analogies: $a$:$x$ :: $b$:$y$ as well as $a$:$y$ :: $b$:$x$, which violates a desideratum of analogies (see \cite{turney2012domain}, section 2.2).  

Related issues are discussed in \cite{turney2012domain, levyGoldberg14}, the latter of which proposes the 3CosMul objective to finding $y$ given $(a,b,x)$:
$$\max_y \frac{(1+\cos(\vec{x},\vec{y}))(1+\cos(\vec{x},\vec{b})}{1+\cos(\vec{y}, \vec{a})+\epsilon}.$$
The additional $\epsilon$ is necessary so that the denominator is positive. This approach is designed for finding a single word $y$ and not directly applicable for the problem of generating both $x$ and $y$ as the objective is not symmetric in $x$ and $y$.

In the spirit of their work, we note that a desired property is that the direction $\vec{a}-\vec{b}$ should be similar (in angle) to the direction $\vec{x}-\vec{y}$ even if the magnitudes differ. Interestingly, given $(a,b,x)$, the $y$ that maximizes $\cos(\vec{a}-\vec{b},\vec{x}-\vec{y})$ is generally an extreme. For instance, for $a=$\em{he} and $b=$\em{she}, for the vast majority of words $x$, the word {\em her} maximizes the expression for $y$. This is due to the fact that the most significant difference between a random word $x$ and the word {\em her} is that {\em her} is likely much more feminine than $x$. Since, from a perceptual point of view it is easier to compare and contrast similar items than very different items, we instead seek $x$ and $y$ that are not semantically similar, which is why our definition is restricted to $\|\vec{x}-\vec{y}\|\leq \delta$.  

As $\delta$ varies from small to large, the analogies vary from generating very similar $x$ and $y$ to very loosely related $x$ and $y$ where their relationship is vague and more ``creative''.

Finally, Figure \ref{fig:analogy_comparison} highlights differences between analogies generated from our approach and the corresponding analogies generated by the first approach mentioned above, namely minimizing:
\begin{equation}\label{eq:cosadd}
\min_{x,y:x\neq a, y\neq b, x\neq y}\|(\vec{a}-\vec{b})-(\vec{x}-\vec{y})\|,
\end{equation}
To compare, we took the first 100 analogies generated using the two approaches that did not have any gender-specific words. We then display the first 10 analogies from each list which do not occur in the other list of 100.

\begin{figure}
\begin{tabular}{ll}
   {\bf Analogies generated using eq.~(\ref{eq:cosadd})}	&{\bf Analogies generated using our approach, eq.~(\ref{eq:analogyS})}\\
   petite-diminutive	&petite-lanky	\\
   seventh inning-eighth inning	&volleyball-football	\\
   seventh-sixth	&interior designer-architect	\\
   east-west	&bitch-bastard	\\
   tripled-doubled	&bra-pants	\\
   breast cancer-cancer	&nurse-surgeon	\\
   meter hurdles-meter dash	&feminine-manly	\\
   thousands-tens	&glamorous-flashy	\\
   eight-seven	&registered nurse-physician	\\
   unemployment rate-jobless rate	&cupcakes-pizzas	\\
\end{tabular}
\caption{\label{fig:analogy_comparison}First 10 different {\em she-he} analogies generated using the parallelogram approach and our approach, from the top 100 {\em she-he} analogies not containing gender specific words. Most of the analogies on the left seem to have little connection to gender. }
\end{figure}

\section{Learning the linear transform}\label{sec:transform}
In the soft debiasing algorithm, we need to solve the following optimization problem.

\[
\min_T  \| (TW)^T(TW)-W^T W \|_F^2 + \lambda \|(TN)^T(TB)\|_F^2.
\]

Let $X = T^T T$, then this is equivalent to the following semi-definite programming problem
\begin{equation}
\label{eq:sdp}
\min_X \| W^T X W - W^T W \|_F^2 + \lambda \| N^T X B\|_F^2 \qquad \mbox{s.t.}  X\succeq 0.
\end{equation}
The first term ensures that the pairwise inner products are preserved and the second term induces the biases of gender neutral words onto the gender subspace to be small. The user-specified parameter $\lambda$ balances the two terms.

Directly solving this SDP optimization problem
is challenging.   
In practice, the dimension of matrix $W$ is in the scale of $300 \times 400,000$. 
The dimensions of the matrices $W^T XW$ and $W^T W$ are $400,000 \times 400,000$,
causing computational and memory issues.
We perform singular value decomposition on $W$, such that $W=U\Sigma V^T$, where $U$ and $V$ are orthogonal matrices and $\Sigma$ is a diagonal matrix.
\begin{equation}
\label{eq:pairdist}
\begin{split}
\|W^T XW-W^TW\|_F^2 &= \|W^T(X-I)W\|_F^2\\ 
&= \|V\Sigma U^T (X-I) U \Sigma V^T\|_F^2 \\
&=  \|\Sigma U^T (X-I) U \Sigma\|_F^2.
\end{split}
\end{equation}
The last equality follows the fact that $V$ is an orthogonal matrix and ($\|VYV^T\|_F^2 = tr(VY^TV^TVYV^T)=tr(VY^TYV^T) = tr(Y^TYV^TV) = tr(Y^TY) = \|Y\|_F^2$.)

Substituting Eq.  \eqref{eq:pairdist} to Eq. \eqref{eq:sdp} gives 
\begin{equation}
\label{eq:sdp2}
\begin{split}
\min_{X} &  \|\Sigma U^T (X-I) U \Sigma\|_F^2 + \lambda \|PXS^T\|_F^2 \quad\quad 
\mbox{s.t.} \ X\succeq 0. \\
\end{split}
\end{equation}
Here $\Sigma U^T (X-I) U \Sigma$ is a $300 \times 300 $ matrix and can be solved efficiently. The solution $T$ is the  debiasing transformation of the word embedding.

\section{Details of gender specific words base set}\label{ap:gender_neutral}
This section gives precise details of how we derived our list of gender neutral words. Note that the choice of gender neutral words is partly subjective. Some words are most often associated with females or males but have exceptions, such as {\em beard} (bearded women), {\em estrogen} (men have small amounts of the hormone estrogen), and {\em rabbi} (reformed Jewish congregations recognize female rabbis). There are also many words that have multiple senses, some of which are gender neutral and others of which are gender specific. For instance, the profession of {\em nursing} is gender neutral while {\em nursing} a baby (i.e., breastfeeding) is only performed by women. 

To derive the base subset of words from w2vNEWS, for each of the 26,377 words in the filtered embedding, we selected words whose definitions include any of the following words in their singular or plural forms: {\em female, male, woman, man, girl, boy, sister, brother, daughter, son, grandmother, grandfather, wife, husband}. Definitions were taken from Wordnet \cite{fellbaum98} (in the case where a word had multiple senses/synsets, we chose the definition whose corresponding lemma had greatest frequency in terms of its count). This list of hundreds of words contains most gender specific words of interest but also contains some gender neutral words, e.g., the definition of {\em mating} is ``the act of pairing a male and female for reproductive purposes.'' Even though the word {\em female} is in the definition, {\em mating} is not gender specific. We went through this list and manually selected those words that were clearly gender specific. Motivated by the application of improving web search, we used a strict definition of gender specificity, so that when in doubt a word was defined to be gender neutral. For instance, clothing words (e.g., the definition of {\em vest} is ``a collarless men's undergarment for the upper part of the body'') were classified as gender neutral since there are undoubtedly people of every gender that wear any given type of clothing. After this filtering, we were left with the following list of 218 gender-specific words (sorted by word frequency):

{\em he, his, her, she, him, man, women, men, woman, spokesman, wife, himself, son, mother, father, chairman, daughter, husband, guy, girls, girl, boy, boys, brother, spokeswoman, female, sister, male, herself, brothers, dad, actress, mom, sons, girlfriend, daughters, lady, boyfriend, sisters, mothers, king, businessman, grandmother, grandfather, deer, ladies, uncle, males, congressman, grandson, bull, queen, businessmen, wives, widow, nephew, bride, females, aunt, prostate cancer, lesbian, chairwoman, fathers, moms, maiden, granddaughter, younger brother, lads, lion, gentleman, fraternity, bachelor, niece, bulls, husbands, prince, colt, salesman, hers, dude, beard, filly, princess, lesbians, councilman, actresses, gentlemen, stepfather, monks, ex girlfriend, lad, sperm, testosterone, nephews, maid, daddy, mare, fiance, fiancee, kings, dads, waitress, maternal, heroine, nieces, girlfriends, sir, stud, mistress, lions, estranged wife, womb, grandma, maternity, estrogen, ex boyfriend, widows, gelding, diva, teenage girls, nuns, czar, ovarian cancer, countrymen, teenage girl, penis, bloke, nun, brides, housewife, spokesmen, suitors, menopause, monastery, motherhood, brethren, stepmother, prostate, hostess, twin brother, schoolboy, brotherhood, fillies, stepson, congresswoman, uncles, witch, monk, viagra, paternity, suitor, sorority, macho, businesswoman, eldest son, gal, statesman, schoolgirl, fathered, goddess, hubby, stepdaughter, blokes, dudes, strongman, uterus, grandsons, studs, mama, godfather, hens, hen, mommy, estranged husband, elder brother, boyhood, baritone, grandmothers, grandpa, boyfriends, feminism, countryman, stallion, heiress, queens, witches, aunts, semen, fella, granddaughters, chap, widower, salesmen, convent, vagina, beau, beards, handyman, twin sister, maids, gals, housewives, horsemen, obstetrics, fatherhood, councilwoman, princes, matriarch, colts, ma, fraternities, pa, fellas, councilmen, dowry, barbershop, fraternal, ballerina}

\section {Questionnaire for generating gender stereotypical words}

\textbf{Task: for each category, please enter 10 or more words, separated by commas.} 
We are looking for a variety of creative answers -- this is a mentally challenging HIT that will make you think.

\begin{itemize}
\item {\bf 10 or more comma-separated words definitionally associated with males.} 

Examples: \emph{dude, menswear, king, penis}, ...

\item {\bf 10 or more comma-separated words definitionally associated with females.}

Examples: \emph{queen, Jane, girl}, ...

\item {\bf 10 or more comma-separated words stereotypically associated with males}

Examples: \emph{football, janitor, cocky}, ...

\item {\bf 10 or more comma-separated words stereotypically associated with females}

Examples: \emph{pink, sewing, caring, sassy, nurse}, ...
\end{itemize}

Thank you for your help in making Artificially Intelligent systems that aren't prejudiced. :-)

\section {Questionnaire for generating gender stereotypical analogies}

An analogy describes two pairs of words where the relationship between the two words in each pair is the same. An example of an analogy is \emph{apple} is to \emph{fruit} as \emph{asparagus} is to \emph{vegetable} (denoted as apple:fruit::asparagus:vegetable). We need your help to improve our analogy generating system.

\paragraph{} 

{\bf Task: please enter 10 or more analogies reflecting gender stereotypes, separated by commas}. We are looking for a variety of creative answers -- this is a mentally challenging HIT that will make you think.

\paragraph{Examples of stereotypes}
\begin{itemize}
\item tall : man :: short : woman reflects a cultural stereotype that men are tall and women are short. 

\item doctor : man :: nurse : woman reflects a stereotype that doctors are typically men and nurses are typically women.
\end{itemize}

\section{Questionnaire for rating stereotypical analogies}

An analogy describes two pairs of words where the relationship between the two words in each pair is the same. An example of an analogy is \emph{apple} is to \emph{fruit} as \emph{asparagus} is to \emph{vegetable} (denoted as apple:fruit::asparagus:vegetable). We need your help to improve our analogy generating system.

\paragraph{}
\noindent {\bf Task: Which analogies are stereotypes?  Which ones are appropriate analogies?}

\begin{itemize}
\item {\bf Examples of stereotype analogies}

tall : man :: short : woman 

doctor : man :: nurse : woman

\item {\bf Examples of appropriate analogies}

King: man :: Queen : woman  

brother : man :: sister : woman  

John : man  :: Mary : woman

His : man  :: Hers : woman 

salesman : man :: saleswoman : woman

penis : man :: vagina : woman

\end{itemize}

WARNING: This HIT may contain adult content. Worker discretion is advised.

Check the analogies that are stereotypes

...

Check the analogies that are nonsensical

...

Check the analogies that are nonsensical

...

Any suggestions or comments on the hit? 
Optional feedback

\section{Analogies Generated by Word Embeddings}
\label{app:analogies}

\begin{longtable}{l|c|c|l|c|c}
\hline
\multicolumn{3}{c|}{After executing hard debiasing} &
\multicolumn{3}{c}{Before executing debiasing} \\
\hline
Analogy & Appropriate & Biased & Analogy &Appropriate &Biased\\
\hline
hostess:bartender & 1 & 8 & midwife:doctor & 1 &10 \\
ballerina:dancer & 0 & 7 & sewing:carpentry & 2 &9 \\
colts:mares & 6 & 7 & pediatrician:orthopedic\_surgeon & 0 &9 \\
ma:na & 8 & 7 & registered\_nurse:physician & 1 &9 \\
salesperson:salesman & 1 & 7 & housewife:shopkeeper & 1 &9 \\
diva:superstar & 4 & 7 & skirts:shorts & 0 &9 \\
witches:vampires & 1 & 7 & nurse:surgeon & 1 &9 \\
hair\_salon:barbershop & 4 & 6 & interior\_designer:architect & 1 &8 \\
maid:housekeeper & 3 & 6 & softball:baseball & 4 &8 \\
soprano:baritone & 4 & 5 & blond:burly & 2 &8 \\
footy:blokes & 0 & 5 & nanny:chauffeur & 1 &8 \\
maids:servants & 4 & 5 & feminism:conservatism & 2 &8 \\
dictator:strongman & 0 & 5 & adorable:goofy & 0 &8 \\
bachelor:bachelor\_degree & 7 & 4 & vocalists:guitarists & 0 &8 \\
witch:witchcraft & 0 & 4 & cosmetics:pharmaceuticals & 1 &8 \\
gaffer:lads & 1 & 3 & whore:coward & 0 &7 \\
convent:monastery & 8 & 3 & vocalist:guitarist & 1 &7 \\
hen:cock & 8 & 2 & petite:lanky & 1 &7 \\
aldermen:councilmen & 0 & 2 & salesperson:salesman & 1 &7 \\
girlfriend:friend & 0 & 2 & sassy:snappy & 2 &7 \\
housewife:homemaker & 2 & 2 & diva:superstar & 4 &7 \\
maternal:infant\_mortality & 1 & 2 & charming:affable & 2 &6 \\
beau:lover & 1 & 2 & giggle:chuckle & 1 &6 \\
mistress:prostitute & 0 & 2 & witch:demon & 2 &6 \\
heroine:protagonist & 2 & 2 & volleyball:football & 1 &6 \\
heiress:socialite & 2 & 2 & feisty:mild\_mannered & 0 &6 \\
teenage\_girl:teenager & 3 & 2 & cupcakes:pizzas & 1 &6 \\
estrogen:testosterone & 9 & 2 & dolls:replicas & 0 &6 \\
actresses:actors & 10 & 1 & netball:rugby & 0 &6 \\
blokes:bloke & 1 & 1 & hairdresser:barber & 6 &5 \\
girlfriends:buddies & 6 & 1 & soprano:baritone & 4 &5 \\
compatriot:countryman & 3 & 1 & gown:blazer & 6 &5 \\
compatriots:countrymen & 2 & 1 & glamorous:flashy & 2 &5 \\
gals:dudes & 10 & 1 & sweater:jersey & 0 &5 \\
eldest:elder\_brother & 1 & 1 & feminist:liberal & 0 &5 \\
sperm:embryos & 2 & 1 & bra:pants & 2 &5 \\
mother:father & 10 & 1 & rebounder:playmaker & 0 &5 \\
wedlock:fathered & 0 & 1 & nude:shirtless & 0 &5 \\
mama:fella & 7 & 1 & judgmental:arrogant & 1 &4 \\
lesbian:gay & 8 & 1 & boobs:ass & 1 &4 \\
kid:guy & 1 & 1 & salon:barbershop & 7 &4 \\
carpenter:handyman & 5 & 1 & lovely:brilliant & 0 &4 \\
she:he & 9 & 1 & practicality:durability & 0 &4 \\
herself:himself & 10 & 1 & singer:frontman & 0 &4 \\
her:his & 10 & 1 & gorgeous:magnificent & 2 &4 \\
uterus:intestine & 1 & 1 & ponytail:mustache & 2 &4 \\
queens:kings & 10 & 1 & feminists:socialists & 0 &4 \\
female:male & 9 & 1 & bras:trousers & 5 &4 \\
women:men & 10 & 1 & wedding\_dress:tuxedo & 6 &4 \\
pa:mo & 9 & 1 & violinist:virtuoso & 0 &4 \\
nun:monk & 7 & 1 & handbag:briefcase & 8 &3 \\
matriarch:patriarch & 9 & 1 & giggling:grinning & 0 &3 \\
nuns:priests & 9 & 1 & kids:guys & 3 &3 \\
menopause:puberty & 2 & 1 & beautiful:majestic & 1 &3 \\
fiance:roommate & 0 & 1 & feminine:manly & 8 &3 \\
daughter:son & 9 & 1 & convent:monastery & 8 &3 \\
daughters:sons & 10 & 1 & sexism:racism & 0 &3 \\
spokeswoman:spokesman & 10 & 1 & pink:red & 0 &3 \\
politician:statesman & 1 & 1 & blouse:shirt & 6 &3 \\
stallion:stud & 7 & 1 & bitch:bastard & 8 &2 \\
suitor:takeover\_bid & 8 & 1 & wig:beard & 4 &2 \\
waitress:waiter & 10 & 1 & hysterical:comical & 0 &2 \\
lady:waitress & 0 & 1 & male\_counterparts:counterparts & 1 &2 \\
bride:wedding & 0 & 1 & beauty:grandeur & 0 &2 \\
widower:widowed & 3 & 1 & cheerful:jovial & 0 &2 \\
husband:younger\_brother & 3 & 1 & breast\_cancer:lymphoma & 3 &2 \\
actress:actor & 9 & 0 & heiress:magnate & 6 &2 \\
mustache:beard & 0 & 0 & estrogen:testosterone & 9 &2 \\
facial\_hair:beards & 0 & 0 & starlet:youngster & 2 &2 \\
suitors:bidders & 6 & 0 & Mary:John & 9 &1 \\
girl:boy & 9 & 0 & actresses:actors & 10 &1 \\
childhood:boyhood & 1 & 0 & middle\_aged:bearded & 0 &1 \\
girls:boys & 10 & 0 & mums:blokes & 5 &1 \\
counterparts:brethren & 4 & 0 & girlfriends:buddies & 6 &1 \\
brides:bridal & 1 & 0 & mammogram:colonoscopy & 0 &1 \\
sister:brother & 10 & 0 & compatriot:countryman & 3 &1 \\
friendship:brotherhood & 3 & 0 & luscious:crisp & 0 &1 \\
sisters:brothers & 9 & 0 & gals:dudes & 10 &1 \\
businesswoman:businessman & 9 & 0 & siblings:elder\_brother & 1 &1 \\
businesspeople:businessmen & 1 & 0 & mother:father & 10 &1 \\
chairwoman:chairman & 10 & 0 & babe:fella & 9 &1 \\
bastard:chap & 0 & 0 & lesbian:gay & 8 &1 \\
hens:chickens & 3 & 0 & breasts:genitals & 0 &1 \\
viagra:cialis & 1 & 0 & wonderful:great & 0 &1 \\
filly:colt & 9 & 0 & she:he & 9 &1 \\
fillies:colts & 8 & 0 & herself:himself & 10 &1 \\
congresswoman:congressman & 9 & 0 & her:his & 10 &1 \\
councilwoman:councilman & 9 & 0 & mommy:kid & 0 &1 \\
wife:cousin & 0 & 0 & queens:kings & 10 &1 \\
mom:dad & 10 & 0 & female:male & 9 &1 \\
mommy:daddy & 10 & 0 & women:men & 10 &1 \\
moms:dads & 9 & 0 & boyfriend:pal & 0 &1 \\
widow:deceased & 0 & 0 & matriarch:patriarch & 9 &1 \\
gal:dude & 9 & 0 & nun:priest & 10 &1 \\
stepmother:eldest\_son & 3 & 0 & breast:prostate & 9 &1 \\
deer:elk & 1 & 0 & daughter:son & 9 &1 \\
estranged\_husband:estranged & 0 & 0 & daughters:sons & 10 &1 \\
ex\_boyfriend:ex\_girlfriend & 7 & 0 & spokeswoman:spokesman & 10 &1 \\
widows:families & 4 & 0 & fabulous:terrific & 3 &1 \\
motherhood:fatherhood & 10 & 0 & headscarf:turban & 6 &1 \\
mothers:fathers & 10 & 0 & waitress:waiter & 10 &1 \\
guys:fellas & 1 & 0 & husband:younger\_brother & 3 &1 \\
feminism:feminist & 1 & 0 & hers:yours & 2 &1 \\
womb:fetus & 0 & 0 & teenage\_girls:youths & 0 &1 \\
sorority:fraternity & 9 & 0 & actress:actor & 9 &0 \\
lesbians:gays & 9 & 0 & blonde:blond & 4 &0 \\
mare:gelding & 7 & 0 & girl:boy & 9 &0 \\
fella:gentleman & 1 & 0 & childhood:boyhood & 1 &0 \\
ladies:gentlemen & 10 & 0 & girls:boys & 10 &0 \\
boyfriends:girlfriend & 3 & 0 & sister:brother & 10 &0 \\
goddess:god & 9 & 0 & sisters:brothers & 9 &0 \\
grandmother:grandfather & 10 & 0 & businesswoman:businessman & 9 &0 \\
grandma:grandpa & 9 & 0 & chairwoman:chairman & 10 &0 \\
grandmothers:grandparents & 5 & 0 & filly:colt & 9 &0 \\
granddaughter:grandson & 10 & 0 & fillies:colts & 8 &0 \\
granddaughters:grandsons & 9 & 0 & congresswoman:congressman & 9 &0 \\
me:him & 2 & 0 & councilwoman:councilman & 9 &0 \\
queen:king & 10 & 0 & mom:dad & 10 &0 \\
youngster:lad & 1 & 0 & moms:dads & 9 &0 \\
elephant:lion & 0 & 0 & gal:dude & 9 &0 \\
elephants:lions & 0 & 0 & motherhood:fatherhood & 10 &0 \\
manly:macho & 4 & 0 & mothers:fathers & 10 &0 \\
females:males & 10 & 0 & sorority:fraternity & 9 &0 \\
woman:man & 8 & 0 & mare:gelding & 7 &0 \\
fiancee:married & 4 & 0 & lady:gentleman & 9 &0 \\
maternity:midwives & 1 & 0 & ladies:gentlemen & 10 &0 \\
monks:monasteries & 0 & 0 & goddess:god & 9 &0 \\
niece:nephew & 9 & 0 & grandmother:grandfather & 10 &0 \\
nieces:nephews & 9 & 0 & grandma:grandpa & 9 &0 \\
hubby:pal & 1 & 0 & granddaughter:grandson & 10 &0 \\
obstetrics:pediatrics & 3 & 0 & granddaughters:grandsons & 9 &0 \\
vagina:penis & 10 & 0 & kinda:guy & 1 &0 \\
princess:prince & 9 & 0 & heroine:hero & 9 &0 \\
colon:prostate & 6 & 0 & me:him & 2 &0 \\
ovarian\_cancer:prostate\_cancer & 10 & 0 & queen:king & 10 &0 \\
salespeople:salesmen & 2 & 0 & females:males & 10 &0 \\
semen:saliva & 7 & 0 & woman:man & 8 &0 \\
schoolgirl:schoolboy & 8 & 0 & niece:nephew & 9 &0 \\
replied:sir & 0 & 0 & nieces:nephews & 9 &0 \\
spokespeople:spokesmen & 0 & 0 & vagina:penis & 10 &0 \\
boyfriend:stepfather & 1 & 0 & princess:prince & 9 &0 \\
stepdaughter:stepson & 9 & 0 & ovarian\_cancer:prostate\_cancer & 10 &0 \\
teenage\_girls:teenagers & 1 & 0 & schoolgirl:schoolboy & 8 &0 \\
hers:theirs & 0 & 0 & spokespeople:spokesmen & 0 &0 \\
twin\_sister:twin\_brother & 9 & 0 & stepdaughter:stepson & 9 &0 \\
aunt:uncle & 9 & 0 & twin\_sister:twin\_brother & 9 &0 \\
aunts:uncles & 10 & 0 & aunt:uncle & 9 &0 \\
husbands:wives & 7 & 0 & aunts:uncles & 10 &0 \\
\end{longtable}

\section{Debiasing the full w2vNEWS embedding.}
In the main text, we focused on the results from a cleaned version of w2vNEWS consisting of 26,377 lower-case words. We have also applied our hard debiasing algorithm to the full w2vNEWS dataset. Evalution based on the standard metrics shows that the debiasing does not degrade the utility of the embedding (Table~\ref{tab:results_full}).  

\begin{table}
\begin{tabular}{l|lll}
  & RG & WS & analogy  \\
  \hline
  \hline\\
Before  &	76.1& 70.0 & 71.2 \\
Hard-debiased & 76.5	& 69.7 & 71.2\\
Soft-debiased & 76.9 & 69.7 & 71.2 \\
\end{tabular}
	\caption{The columns show the performance of the original, complete 
	w2vNEWS embedding (``before'') and the debiased w2vNEWS on the standard evaluation metrics measuring coherence and analogy-solving abilities: RG \cite{rubenstein1965contextual}, WS \cite{finkelstein2001placing}, MSR-analogy \cite{mikolov2013linguistic}. Higher is better. The results show that the performance does not degrade after debiasing. 
	\label{tab:results_full}}
\end{table}

\end{document}